\let\accentvec\vec            
\documentclass[runningheads]{llncs}
\let\vec\accentvec

\usepackage[utf8]{inputenc} 
\usepackage[T1]{fontenc}    
\usepackage{lmodern}
\usepackage{url}            
\usepackage{booktabs}       
\usepackage{amsfonts}       
\usepackage{nicefrac}       
\usepackage{microtype}      

\newcommand{\citep}{\cite}
\newcommand{\citet}{\cite}

\usepackage{graphicx}
\usepackage[cmex10]{amsmath}
\usepackage{amssymb}
\usepackage{hyperref}
\usepackage{cleveref}
\usepackage[usenames,dvipsnames]{xcolor}
\usepackage{float}

\usepackage{calc}
\usepackage{bbm}
\usepackage{mathtools}
\usepackage{bbm}
\usepackage{multirow}
\usepackage[linesnumbered,ruled,vlined]{algorithm2e}
\usepackage{todonotes}
\usepackage{svg}
\usepackage{subfig}

\newcommand{\Tau}{\mathcal{T}}
\newcommand{\LL}{\mathcal{L}}
\newcommand{\dquad}{d_\texttt{QUAD}}
\newcommand{\dber}{d_\texttt{BER}}
\newcommand{\OLOP}{\texttt{OLOP}\xspace}
\newcommand{\ODP}{\texttt{ODP}\xspace}
\newcommand{\KLOLOP}{\texttt{KL-OLOP}\xspace}

\DeclarePairedDelimiter{\ceil}{\lceil}{\rceil}
\DeclarePairedDelimiter{\floor}{\lfloor}{\rfloor}

\renewcommand{\epsilon}{\varepsilon}
\renewcommand{\tilde}{\widetilde}

\newcommand{\eqdef}{\buildrel \text{def}\over =}

\def\:#1{\protect \ifmmode {\mathbf{#1}} \else {\textbf{#1}} \fi}

\renewcommand{\epsilon}{\varepsilon}

\DeclareMathOperator*{\argmax}{arg\,max}

\newcommand{\probability}[1]{\mathbb{P}\left(#1\right)}

\DeclareMathOperator*{\expectedvalue}{\mathbb{E}}

\newcommand{\diff}[1]{\textcolor{ForestGreen}{#1}}

\begin{document}

\title{Practical Open-Loop Optimistic Planning}
%
%
\author{Edouard Leurent\inst{1,2}\and
Odalric-Ambrym Maillard\inst{1}}
\authorrunning{E. Leurent, O-A. Maillard}
%
\institute{SequeL team, INRIA Lille - Nord Europe, France\\
\email{\{edouard.leurent,odalric.maillard\}@inria.fr}\and
Renault Group, France\\
\email{edouard.leurent@renault.com}}
\maketitle              

\begin{abstract}
    We consider the problem of online planning in a Markov Decision Process when given only access to a generative model, restricted to open-loop policies - i.e. sequences of actions - and under budget constraint. In this setting, the \textit{Open-Loop Optimistic Planning} (\OLOP) algorithm enjoys good theoretical guarantees but is overly conservative in practice, as we show in numerical experiments. We propose a modified version of the algorithm with tighter upper-confidence bounds, \KLOLOP, that leads to better practical performances while retaining the sample complexity bound. Finally, we propose an efficient implementation that significantly improves the time complexity of both algorithms.
    
\end{abstract}

\section{Introduction}

In a \emph{Markov Decision Process} (MDP), an agent observes its current state $s$ from a state space $S$ and picks an action $a$ from an action space $A$, before transitioning to a next state $s'$ drawn from a transition kernel $\probability{s'|s,a}$ and receiving a bounded reward $r\in[0, 1]$ drawn from a reward kernel $\probability{r|s, a}$. The agent must act so as to optimise its expected cumulative discounted reward $\expectedvalue \sum_t \gamma^t r_t$, also called expected \emph{return}, where $\gamma\in[0,1)$ is the discount factor. In \emph{Online Planning} \cite{Munos2014}, we do not consider that these transition and reward kernels are known as in \emph{Dynamic Programming} \citep{Bellman1957}, but rather only assume access to the MDP through a \emph{generative model} (e.g. a simulator) which yields samples of the next state $s' \sim \probability{s'|s,a}$ and reward $r\sim\probability{r|s, a}$ when queried. Finally, we consider a \emph{fixed-budget} setting where the generative model can only be called a maximum number of times, called the budget $n$. 

\emph{Monte-Carlo Tree Search} (\texttt{MCTS}) algorithms were historically motivated by the application of computer Go, and made a first appearance in the CrazyStone software \citet{Coulom2006}. They were later reformulated in the setting of Multi-Armed Bandits by \citet{Kocsis2006} with their \emph{Upper Confidence bounds applied to Trees} (\texttt{UCT}) algorithm. Despite its popularity, \texttt{UCT} has been shown to suffer from several limitations: its sample complexity can be at least doubly-exponential for some problems (e.g. when a narrow optimal path is hidden in a suboptimal branch), which is much worse than uniform planning \citep{Coquelin2007}. The \texttt{Sparse Sampling} algorithm of \citet{Kearns2002} achieves better worst-case performance, but it is still non-polynomial and doesn't adapt to the structure of the MDP. In stark contrast, the \emph{Optimistic Planning for Deterministic systems} (\texttt{OPD}) algorithm considered by \citet{Hren2008} in the case of deterministic transitions and rewards exploits the structure of the cumulative discounted reward to achieve a problem-dependent polynomial bound on sample complexity. A similar line of work in a deterministic setting is that of \texttt{SOOP} and \texttt{OPC} by \cite{Busoniu2013,Busoniu2018} though they focus on continuous action spaces. \texttt{OPD} was later extended to stochastic systems with the \emph{Open-Loop Optimistic Planning} (\OLOP) algorithm introduced by \citet{Bubeck2010} in the open-loop setting: we only consider sequences of actions independently of the states that they lead to. This restriction in the space of policies causes a loss of optimality, but greatly simplifies the planning problem in the cases where the state space is large or infinite. More recent work such as \texttt{St0p} \citep{Szorenyi2014} and \texttt{TrailBlazer} \citep{Grill2016} focus on the probably approximately correct (PAC) framework: rather than simply recommending an action to maximise the expected rewards, they return an $\epsilon$-approximation of the value at the root that holds with high probability. This highly demanding framework puts a severe strain on these algorithms that were developed for theoretical analysis only and cannot be applied to real problems.

\paragraph{Contributions} The goal of this paper is to study the practical performances of \OLOP when applied to numerical problems. Indeed, \OLOP was introduced along with a theoretical sample complexity analysis but no experiment was carried-out. Our contribution is threefold:

\begin{itemize}
    \item First, we show that in our experiments \OLOP is overly pessimistic, especially in the low-budget regime, and we provide an intuitive explanation by casting light on an unintended effect that alters the behaviour of \OLOP.
    \item Second, we circumvent this issue by leveraging modern tools from the bandits literature to design and analyse a modified version with tighter upper-confidence bounds called \KLOLOP. We show that we retain the asymptotic regret bounds of $\OLOP$ while improving its performances by an order of magnitude in numerical experiments.
    \item Third, we provide a time and memory efficient implementation of \OLOP and \KLOLOP, bringing an exponential speedup that allows to scale these algorithms to high sample budgets.
\end{itemize}

The paper is structured as follows: in section \ref{sec:kl-olop}, we present \OLOP, give some intuition on its limitations, and introduce \KLOLOP, whose sample complexity is further analysed in section \ref{sec:sample-complexity}. In section \ref{sec:time-complexity}, we propose an efficient implementation of the two algorithms. Finally in section \ref{sec:experiments}, we evaluate them in several numerical experiments.

\subsubsection{Notations}
Throughout the paper, we follow the notations from \citep{Bubeck2010} and use the standard notations over alphabets: a finite word $a \in A^*$ of length $h$ represents a sequence of actions $(a_0, \cdots, a_h) \in A^h$. Its prefix of length $t \leq h$ is denoted $a_{1:t} = (a_0,\cdots,a_t) \in A^t$. $A^\infty$ denotes the set of infinite sequences of actions. Two finite sequences $a\in A^*$ and $b\in A^*$ can be concatenated as $ab\in A^*$, the set of finite and infinite suffixes of $a$ are respectively $a A^* = \{c\in\mathcal{A}^*: \exists b\in A^*$ such that $c=ab\}$ and $aA^\infty$ defined likewise, and the empty sequence is $\emptyset$.

During the planning process, the agent iteratively selects sequences of actions until it reaches the allowed budget of $n$ actions. More precisely, at time $t$ during the $m^{\text{th}}$ sequence, the agent played $a^m_{1:t} = a^m_1 \cdots a^m_t \in A^t$ and receives a reward $Y_t^m$. We denote the probability distribution of this reward as $\nu(a_{1:t}^m) = \probability{Y_t^m | s_t, a^m_t} \prod_{k=1}^{t-1} \probability{s_{k+1} | s_{k}, a_{k}^m}$, and its mean as $\mu(a_{1:t}^m)$, where $s_1$ is the current state.

After this exploration phase, the agent selects an action $a(n)$ so as to minimise the \emph{simple regret} $r_n = V - V(a(n))$, where $V=V(\emptyset)$ and $V(a)$ refers to the value of a sequence of actions $a\in A^h$, that is, the maximum expected discounted cumulative reward one may obtain after executing $a$:

\begin{equation}
\label{eq:value}
V(a) = \sup_{b\in aA^\infty} \sum_{t=1}^\infty \gamma^t\mu(b_{1:t}),
\end{equation}

\section{Kullback-Leibler Open-Loop Optimistic Planning}
\label{sec:kl-olop}

In this section we present \KLOLOP, a combination of the \OLOP algorithm of \citep{Bubeck2010} with the tighter Kullback-Leibler upper confidence bounds from \citep{Cappe2013}. We first frame both algorithms in a common structure before specifying their implementations.

\subsection{General structure}

First, following \OLOP, the total sample budget $n$ is split in $M$ trajectories of length $L$ in the following way: 
\begin{align*}
    &M\text{ is the largest integer such that } M \ceil{\log M/(2 \log 1/\gamma)} \leq n;\\
    &L= \ceil{\log M / (2 \log 1/\gamma)}.
\end{align*}
The look-ahead tree of depth $L$ is denoted $\Tau = \sum_{h=0}^L A^h$.

Then, we introduce some useful definitions. Consider episode $1 \leq m \leq M$. For any $1 \leq h \leq L$ and $a\in A^h$, let

\begin{equation*}
    T_a(m) \eqdef \sum_{s=1}^m \mathbbm{1}\{a^s_{1:h} = a\}
\end{equation*}

\noindent
be the number of times we played an action sequence starting with $a$, and $S_a(m)$ the sum of rewards collected at the last transition of the sequence $a$:

\begin{equation*}
    S_a(m) \eqdef \sum_{s=1}^m Y^s_h \mathbbm{1}\{a^s_{1:h} = a\}
\end{equation*}

\noindent
The empirical mean reward of $a$ is
$\quad\displaystyle{ \hat{\mu}_a(m) \eqdef \frac{S_a(m)}{T_a(m)}} \quad $
if $T_a(m) > 0$, and $+\infty$ otherwise. Here, we provide a more general form for upper and lower confidence bounds on these empirical means:
\begin{eqnarray}
\label{eq:u_mu_a_m}
    U^{\mu}_a(m) &\eqdef \max \left\{q\in I: T_a(m) d(\frac{S_a(m)}{T_a(m)}, q) \leq f(m) \right\}\\
    L^{\mu}_a(m) &\eqdef \min \left\{q\in I: T_a(m) d(\frac{S_a(m)}{T_a(m)}, q) \leq f(m) \right\}
\end{eqnarray}
where $I$ is an interval, $d$ is a divergence on $I\times I \rightarrow \mathbb{R^+}$ and $f$ is a non-decreasing function. They are left unspecified for now and their particular implementations and associated properties will be discussed in the following sections.

These upper-bounds $U^{\mu}_a$ for intermediate rewards finally enable us to define an upper bound $U_a$ for the value $V(a)$ of the entire sequence of actions $a$:

\begin{equation}
\label{eq:Ua}
    U_a(m) \eqdef \sum_{t=1}^h \gamma^t U^{\mu}_{a_{1:t}}(m) + \frac{\gamma^{h+1}}{1-\gamma}
\end{equation}

\noindent
where $\frac{\gamma^{h+1}}{1-\gamma}$ comes from upper-bounding by one every reward-to-go in the sum \eqref{eq:value}, for $t\geq h+1$. In \citep{Bubeck2010}, there is an extra step to "sharpen the bounds" of sequences $a \in A^L$ by taking:

\begin{equation}
\label{eq:Ba}
    B_a(m) \eqdef \inf_{1 \leq t \leq L} U_{a_{1:t}}(m)
\end{equation}

The general algorithm structure is shown in Algorithm \ref{algo:kl-olop}.
We now discuss two specific implementations that differ in their choice of divergence $d$ and non-decreasing function $f$. They are compared in Table \ref{tab:comparison}.

\begin{algorithm}[tp]
\DontPrintSemicolon
\For{each episode $m = 1, \cdots, M$}{
Compute $U_a(m-1)$ from \eqref{eq:Ua} for all $a\in\Tau$\;
Compute $B_a(m-1)$ from \eqref{eq:Ba} for all $a\in A^L$\;\label{alg:b_values_compute}
Sample a sequence with highest B-value: $a^m \in \argmax_{a\in A^L} B_a(m-1)$\;
}
\Return the most played sequence $a(n) \in \argmax_{a\in A^L} T_a(M)$
\caption{General structure for Open-Loop Optimistic Planning}
\label{algo:kl-olop}
\end{algorithm}

\begin{table}[tp]
    \caption{Different implementations of Algorithm \ref{algo:kl-olop} in \OLOP and \KLOLOP}
    \label{tab:comparison}
    \centering
    \begin{tabular}{ccc}
    \toprule
        Algorithm & \OLOP & \KLOLOP \\
        \midrule
        Interval $I$ & $\mathbb{R}$ & [0, 1] \\
        Divergence $d$ & $\dquad$ & $\dber$ \\
        $f(m)$ & $4 \log M$ & $2\log M + 2 \log\log M$\\
        \bottomrule
    \end{tabular}
\end{table}

\subsection{OLOP}
\label{sec:kl-olop-olop}
To recover the original \OLOP algorithm of \citet{Bubeck2010} from Algorithm \ref{algo:kl-olop}, we can use a quadratic divergence $\dquad$ on $I=\mathbb{R}$ and a constant function $f_4$ defined as follows:
\begin{equation*}
    \dquad(p,q) \eqdef 2(p-q)^2,\qquad
    f_4(m) \eqdef 4 \log M
\end{equation*}
Indeed, in this case $U^{\mu}_a(m)$ can then be explicitly computed as:
\begin{align*}
    U^{\mu}_a(m) &= \max \left\{q\in \mathbb{R}: 2(\frac{S_a(m)}{T_a(m)} - q)^2 \leq \frac{4 \log M }{T_a(m)} \right\} = \hat{\mu}_a(m) + \sqrt{\frac{2 \log M}{T_a(m)}}
\end{align*}
which is the Chernoff-Hoeffding bound used originally in section 3.1 of \cite{Bubeck2010}.

\subsection{An unintended behaviour}
\label{sec:kl-olop-behaviour}
From the definition of $U_a(m)$ as an upper-bound of the value of the sequence $a$, we expect increasing sequences $(a_{1:t})_t$ to have non-increasing upper-bounds. Indeed, every new action $a_t$ encountered along the sequence is a potential loss of optimality.
However, this property is only true if the upper-bound defined in \eqref{eq:u_mu_a_m} belongs to the reward interval $[0,1]$.

\begin{lemma}(Monotony of $U_a(m)$ along a sequence)
\label{lemma:seq_values}

\begin{itemize}
    \item If it holds that $U^{\mu}_b(m) \in [0, 1]$ for all $b\in A^*$, then for any $a\in A^L$ the sequence $(U_{a_{1:h}}(m))_{1\leq h \leq L}$ is non-increasing, and we simply have $B_a(m) = U_a(m)$.
    \item Conversely, if $U^{\mu}_b(m) > 1$ for all $b\in A^*$, then for any $a\in A^L$ the sequence $(U_{a_{1:h}}(m))_{1\leq h \leq L}$ is non-decreasing, and we have $B_a(m) = U_{a_{1:1}}(m)$.
\end{itemize}
\end{lemma}

\begin{proof}
We prove the first proposition, and the same reasoning applies to the second. For $a\in A^L$ and $1 \leq h \leq L - 1$, we have by \eqref{eq:Ua}:

\begin{align*}
    U_{a_{1:h+1}}(m) - U_{a_{1:h}}(m) &= \gamma^{h+1}U^{\mu}_{a_{1:h+1}}(m) + \frac{\gamma^{h+2}}{1-\gamma} - \frac{\gamma^{h+1}}{1-\gamma}\\
    &= \gamma^{h+1}(\underbrace{U^{\mu}_{a_{1:h+1}}(m)}_{\in [0, 1]} - 1) \leq 0
\end{align*}

\noindent
We can conclude that $(U_{a_{1:h}}(m))_{1\leq h \leq L}$ is non-increasing and that $B_a(m) = \inf_{1 \leq h \leq L} U_{a_{1:h}}(m) = U_{a_{1:L}}(m) = U_a(m)$.
\qed
\end{proof}

Yet, the Chernoff-Hoeffding bounds used in \OLOP start in the $U^{\mu}_a(m) > 1$ regime -- initially $U^{\mu}_a(m) = \infty$ -- and can remain in this regime for a long time especially in the near-optimal branches where $\hat{\mu}_a(m)$ is close to one.

Under these circumstances, the Lemma \ref{lemma:seq_values} has a drastic effect on the search behaviour. Indeed, as long as a subtree under the root verifies $U^{\mu}_a(m) > 1$ for every sequence $a$, then all these sequences share the same B-value $B_a(m) = U_{a_{1:1}(m)}$. This means that \OLOP cannot differentiate them and exploit information from their shared history as intended, and behaves as uniform sampling instead.
Once the early depths have been explored sufficiently, \OLOP resumes its intended behaviour, but the problem is only shifted to deeper unexplored subtrees.
 
 This consideration motivates us to leverage the recent developments in the Multi-Armed Bandits literature, and modify the upper-confidence bounds for the expected rewards $U^\mu_a(m)$ so that they respect the reward bounds.

\subsection{KL-OLOP}
\label{sec:kl-olop-kl-olop}

\noindent
We propose a novel implementation of Algorithm \ref{algo:kl-olop} where we leverage the analysis of the kl-UCB algorithm from \citep{Cappe2013} for multi-armed bandits with general bounded rewards.
Likewise, we use the Bernoulli Kullback-Leibler divergence defined on the interval $I=[0,1]$ by:

\begin{equation*}
    \dber(p, q) \eqdef p \log \frac{p}{q} + (1-p)\log\frac{1-p}{1-q}
\end{equation*}

\noindent
with, by convention, $0 \log 0 = 0 \log 0/0 = 0$ and $x \log x /0 = +\infty$ for $x>0$.
This divergence and the corresponding bounds are illustrated in Figure \ref{fig:ukl}.

\begin{figure}[tp]
    \centering
    \includegraphics[width=0.6\textwidth]{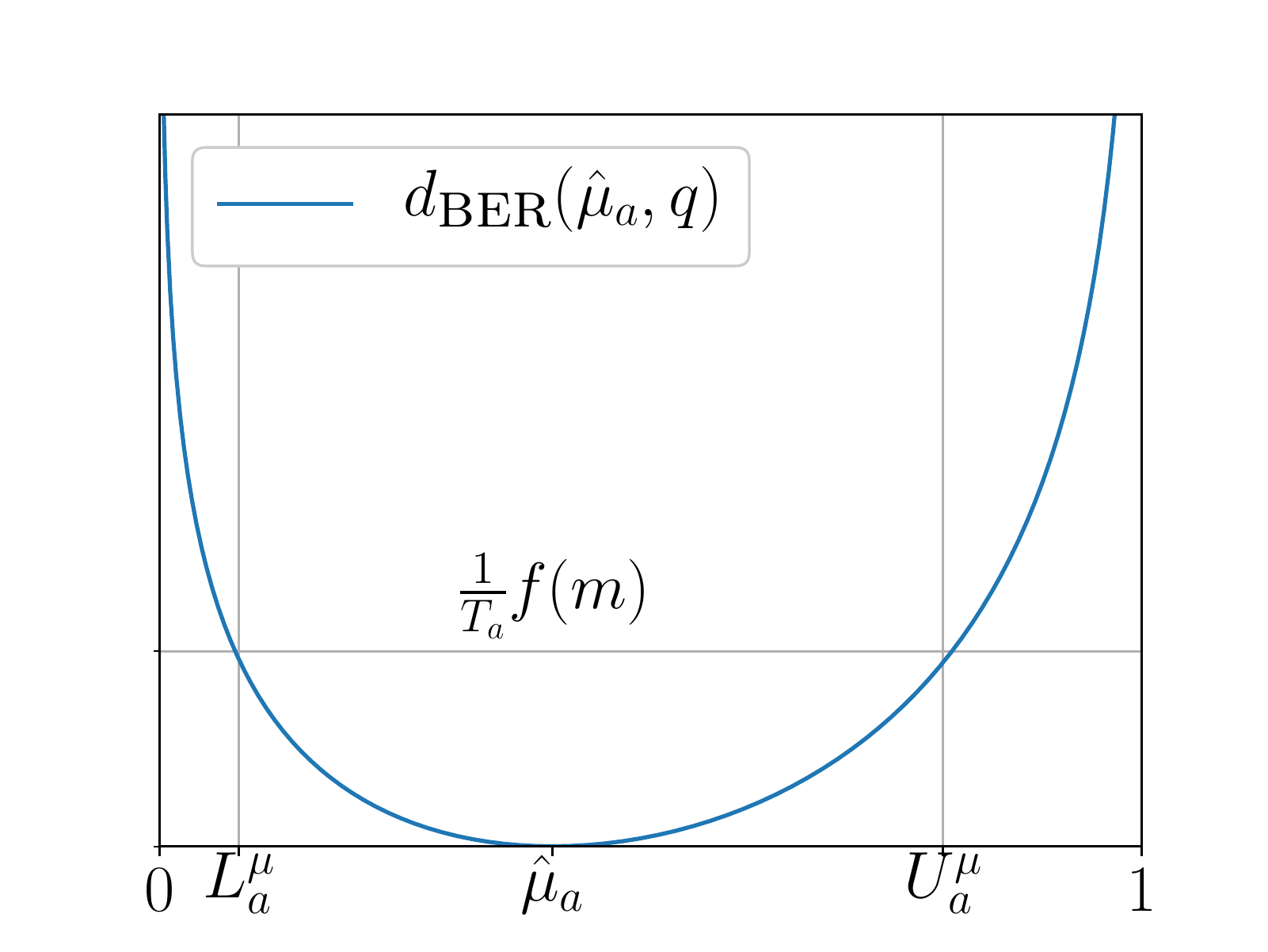}
    \caption{The Bernoulli Kullback-Leibler divergence $\dber$, and the corresponding upper and lower confidence bounds $U^{\mu}_a$ and $L^{\mu}_a$ for the empirical average $\hat{\mu_a}$. Lower values of $f(m)$ give tighter confidence bounds that hold with lower probabilities.}
    \label{fig:ukl}
\end{figure}

$U^{\mu}_a(m)$ and $L^{\mu}_a(m)$ can be efficiently computed using Newton iterations, as for any $p\in[0, 1]$ the function $q \rightarrow \dber(p,q)$ is strictly convex and increasing (resp. decreasing) on the interval [p, 1] (resp. [0, p]).

Moreover, we use the constant function $f_2: m \rightarrow 2 \log M + 2 \log\log M$. This choice is justified in the end of section \ref{sec:regret-proof}. Because $f_2$ is lower than $f_4$, the Figure \ref{fig:ukl} shows that the bounds are tighter and hence less conservative than that of \OLOP, which should increase the performance, provided that their associated probability of violation does not invalidate the regret bound of \OLOP.

\begin{remark}[Upper bounds sharpening]
\label{rmk:sharpen}
The introduction of the B-values $B_a(m)$ was made necessary in \OLOP by the use of Chernoff-Hoeffding confidence bounds which are not guaranteed to belong to [0, 1]. On the contrary, we have in \KLOLOP that $U^\mu_a(m) \in I = [0,1]$ by construction. By Lemma \ref{lemma:seq_values}, the upper bounds sharpening step in line \ref{alg:b_values_compute} of Algorithm \ref{algo:kl-olop} is now superfluous as we trivially have $B_a(m) = U_a(m)$ for all $a\in A^L$.
\end{remark}

\section{Sample complexity}
\label{sec:sample-complexity}

We say that $u_n = \tilde{O}(v_n)$ if there exist $\alpha, \beta >0$ such that $u_n \leq \alpha \log(v_n)^\beta v_n$.
Let us denote the proportion of near-optimal nodes $\kappa_2$ as:

\begin{equation*}
    \label{eq:kappa}
    \kappa_2 \eqdef \limsup_{h\rightarrow\infty}{\left|\left\{a\in a^H:V(a) \geq V - 2\frac{\gamma^{h+1}}{1-\gamma}\right\}\right|^{1/h}}
\end{equation*}

\begin{theorem}[Sample complexity]
\label{thm:regret}
We show that \KLOLOP enjoys the same asymptotic regret bounds as \OLOP. More precisely, for any $\kappa' > \kappa_2$, \KLOLOP satisfies:

\begin{equation*}
    \expectedvalue r_n = \begin{cases}
      \tilde{0}\left(n^{-\frac{\log 1/\gamma}{\log \kappa'}}\right), & \text{if}\ \gamma\sqrt{\kappa'} > 1 \\
      \tilde{0}\left(n^{-\frac{1}{2}}\right), & \text{if}\ \gamma\sqrt{\kappa'} \leq 1
    \end{cases}
\end{equation*}
\end{theorem}

\section{Time and memory complexity}
\label{sec:time-complexity}

After having considered the sample efficiency of \OLOP and \KLOLOP, we now turn to study their time and memory complexities. We will only mention the case of \KLOLOP for ease of presentation, but all results easily extend to \OLOP.

The Algorithm \ref{algo:kl-olop} requires, at each episode, to compute and store in memory of the reward upper-bounds and U-values of all nodes in the tree $\Tau = \sum_{h=0}^L A^h$.
Hence, its time and memory complexities are 
\begin{equation}
    C(\KLOLOP) = O(M|\Tau|) = O(MK^L).
\end{equation}

The curse of dimensionality brought by the branching factor $K$ and horizon $L$ makes it intractable in practice to actually run \KLOLOP in its original form even for small problems. However, most of this computation and memory usage is wasted, as with reasonable sample budgets $n$ the vast majority of the tree $\Tau$ will not be actually explored and hence does not hold any valuable information.

We propose in Algorithm \ref{algo:lazy-kl-olop} a lazy version of \KLOLOP which only stores and processes the explored subtree, as shown in Figure \ref{fig:tree}, while preserving the inner workings of the original algorithm.

\begin{figure}[ht]
    \centering
    \includegraphics[width=0.6\textwidth]{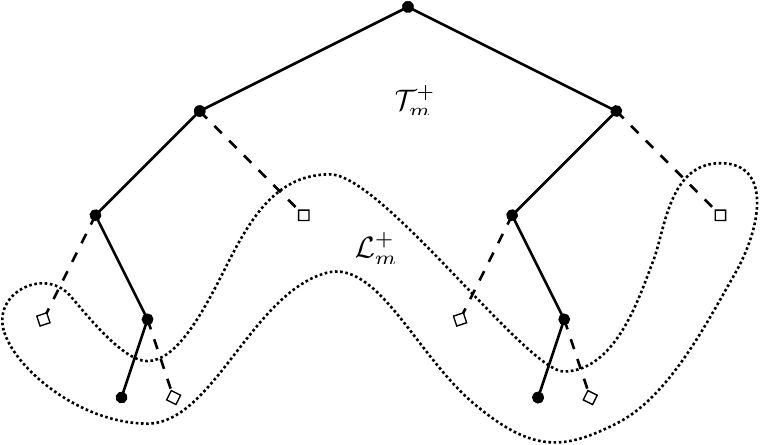}
    \caption{A representation of the tree $\Tau_m^+$, with $K = 2$ actions and after episode $m = 2$, when two sequences have been sampled. They are represented with solid lines and dots \textbullet, and they constitute the explored subtree $\Tau_m$. When extending $\Tau_m$ with the missing children of each node, represented with dashed lines and diamonds $\diamond$, we obtain the full extended subtree $\Tau_m^+$. The set of its leaves is denoted $\LL_m^+$ and shown as a dotted set.}
    \label{fig:tree}
\end{figure}

\begin{algorithm}[tp]
\DontPrintSemicolon
Let $M$ be the largest integer such that $M \log M/(2 \log 1/\gamma) \leq n$\;
Let $L = \log M / (2 \log 1/\gamma)$\;
Let $\Tau_0^+ = \LL_0^+ = \{\emptyset\}$\;
\For{each episode $m = 1, \cdots, M$}{
Compute $U_a(m-1)$ from \eqref{eq:Ua} for all $a\in\Tau_{m-1}^+$\;
Compute $B_a(m-1)$ from \eqref{eq:Ba} for all $a\in \LL_{m-1}^+$\;
Sample a sequence with highest B-value: $a \in \argmax_{a\in \LL_{m-1}^+} B_a(m-1)$\;
Choose an arbitrary continuation $a^m \in aA^{L-|a|}$\tcp*{e.g. uniformly}
Let $\Tau_m^+ = \Tau_{m-1}^+$ and $\LL_m^+ = \LL_{m-1}^+$\;
\For{$t=1, \cdots, L$}{
    \If{$a^m_{1:t} \not \in \Tau_{m}^+$}{
    Add $a^m_{1:t-1}A$ to $\Tau_{m}^+$ and $\LL_{m}^+$\;
    Remove $a^m_{1:t-1}$ from $\LL_{m}^+$
    }
}
}
\Return the most played sequence $a(n) \in \argmax_{a\in \LL_m^+} T_a(M)$
\caption{Lazy Open Loop Optimistic Planning}
\label{algo:lazy-kl-olop}
\end{algorithm}

\begin{theorem}[Consistency]
\label{thm:consistency}
The set of sequences returned by Algorithm \ref{algo:lazy-kl-olop} is the same as the one returned by Algorithm \ref{algo:kl-olop}.
In particular, Algorithm \ref{algo:lazy-kl-olop} enjoys the same regret bounds as in Theorem \ref{thm:regret}.
\end{theorem}

\begin{property}[Time and memory complexity]
Algorithm \ref{algo:lazy-kl-olop} has time and memory complexities of:
\begin{equation*}
    C(\texttt{Lazy KL-OLOP}) = O(KLM^2)
\end{equation*}

The corresponding complexity gain compared to the original Algorithm \ref{algo:kl-olop} is: 
\begin{equation*}
    \frac{C(\texttt{Lazy KL-OLOP})}{C(\KLOLOP)} = \frac{n}{K^{L-1}}
\end{equation*}
which highlights that only a subtree corresponding to the sample budget $n$ is processed instead of the search whole tree $\Tau$.
\end{property}
\begin{proof}
At episode $m = 1, \cdots, M$, we compute and store in memory of the reward upper-bounds and U-values of all nodes in the subtree $\Tau_m^+$. Moreover, the tree $\Tau_m^+$ is constructed iteratively by adding K nodes at most L times at each episode from 0 to $m$. Hence, $|\Tau_m^+| = O(mKL)$.
This yields directly $C(\texttt{Lazy KL-OLOP}) = \sum_{m=1}^M O(mKL) = O(M^2KL)$.
\qed
\end{proof}

\section{Proof of Theorem \ref{thm:regret}}
\label{sec:regret-proof}

We follow step-by step the pyramidal proof of \citep{Bubeck2010}, and adapt it to the Kullback-Leibler upper confidence bound. The adjustments resulting from the change of confidence bounds are \diff{highlighted}. The proofs of lemmas which are not significantly altered are listed in the Supplementary Material. 

We start by recalling their notations.
Let $1 \leq H \leq L$ and $a^* \in A^L$ such that $V(a^*) = V$.
Considering sequences of actions of length $1 \leq h \leq H$, we define the subset $\mathcal{I}_h$ of near-optimal sequences and the subset $\mathcal{J}$ of sub-optimal sequences that were near-optimal at depth $h-1$:
\begin{equation*}
    \mathcal{I}_h = \left\{a \in A^h: V - V(a) \leq 2\frac{\gamma^{h+1}}{1-\gamma}\right\}, \mathcal{J}_h = \left\{a \in A^h: a_{1:h-1} \in \mathcal{I}_{h-1} \text{ and } a \not\in \mathcal{I}_h\right\}
\end{equation*}

By convention, $\mathcal{I}_0 = \{\emptyset\}$. From the definition of $\kappa_2$, we have that for any $\kappa'>\kappa_2$, there exists a constant C such that for any $h \geq 1$,
\begin{equation*}
    |\mathcal{I}_h| \leq C {\kappa'}^h
\end{equation*}
Hence, we also have $|\mathcal{J}_h| \leq K|\mathcal{I}_{h-1}| = O({\kappa'}^h)$.

Now, for $1\leq m \leq M$, $a \in A^t$ with $t \leq h$, $h'<h$, we define the set $\mathcal{P}^a_{h,h'}(m)$ of suffixes of $a$ in $\mathcal{J}_h$ that have been played at least a certain number of times:
\begin{equation*}
    \mathcal{P}^a_{h,h'}(m) = \left\{ b\in a A^{h-t}\cap \mathcal{J}_h : T_b(m) \geq \diff{2f(m)}(h+1)^2\gamma^{2(h'-h+1)} + 1 \right\}
\end{equation*}

and the random variable:

\begin{equation*}
    \tau^a_{h,h'}(m) = \mathbbm{1}\{T_a(m-1) < \diff{2f(m)}(h+1)^2\gamma^{2(h'-h+1)} + 1 \leq T_a(m)\}
\end{equation*}

\begin{lemma}[Regret and sub-optimal pulls]
\label{lemma:expected-regret}
The following holds true:
\begin{equation*}
    r_n \leq \frac{2K \gamma^{H+1}}{1-\gamma} +\frac{3K}{M}\sum_{h=1}^H\sum_{a\in\mathcal{J}_h}\frac{\gamma^h}{1-\gamma}T_a(M)
\end{equation*}
\end{lemma}

The rest of the proof is devoted to the analysis of the term $\expectedvalue \sum_{a\in \mathcal{J}_h} T_a(M)$. The next lemma describes under which circumstances a suboptimal sequence of actions in $\mathcal{J}_h$ can be selected.

\begin{lemma}[Conditions for sub-optimal pull]
\label{lemma:sub-optimal-pull}
Assume that at step $m+1$ we select a sub-optimal sequence $a^{m+1}$: there exist $0 \leq h \leq L,  a\in \mathcal{J}_h$ such that $a^{m+1} \in aA^*$. Then, it implies that one of the following propositions is true:
\begin{equation}
\label{eq:cond-ukl}
\tag{UCB violation}
    \diff{U_{a^*}(m)} < V,
\end{equation}
or
\begin{equation}
\tag{LCB violation}
\label{eq:cond-lkl}
    \sum_{t=1}^h \gamma^t \diff{L^{\mu}_{a_{1:t}}(m)} \geq V(a),
\end{equation}
or
\begin{equation}
\tag{Large CI}
\label{eq:cond-dkl}
    \sum_{t=1}^h \gamma^t\diff{\left(U^{\mu}_{a_{1:t}}(m) - L^{\mu}_{a_{1:t}}(m)\right)} > \frac{\gamma^{h+1}}{1-\gamma}
\end{equation}
\end{lemma}
\begin{proof}
As $a^{m+1}_{1:h} = a$ and \diff{because the U-values are monotonically increasing along sequences of actions} (see Remark \ref{rmk:sharpen} and Lemma \ref{lemma:seq_values}), we have $U_a(m) \geq U_{a^{m+1}}(m)$. Moreover, by Algorithm \ref{algo:kl-olop}, we have $a^{m+1} = \argmax_{a \in A^L}  U_a(m)$ and $a^*\in A^L$, so $U_{a^{m+1}}(m) \geq U_{a^*}(m)$ and finally $U_a(m) \geq U_{a^*}(m)$.

Assume that \eqref{eq:cond-ukl} is false, then:
\begin{equation}
\label{eq:ukl-verifie}
    \sum_{t=1}^h \gamma^t U^{\mu}_{a_{1:t}}(m) + \frac{\gamma^{h+1}}{1-\gamma} = U_a(m) \geq U_{a^*}(m) \geq V
\end{equation}
Assume that \eqref{eq:cond-lkl} is false, then:
\begin{equation}
\label{eq:lkl-verifie}
    \sum_{t=1}^h \gamma^t L^{\mu}_{a_{1:t}}(m) < V(a),
\end{equation}
By taking the difference \eqref{eq:ukl-verifie} - \eqref{eq:lkl-verifie}, 
\begin{equation*}
    \sum_{t=1}^h \gamma^t \left(U^{\mu}_{a_{1:t}}(m) - L^{\mu}_{a_{1:t}}(m)\right) + \frac{\gamma^{h+1}}{1-\gamma} > V - V(a)
\end{equation*}
But $a \in \mathcal{J}_h$, so $V - V(a) \geq \frac{2\gamma^{h+1}}{1-\gamma}$, which yields \eqref{eq:cond-dkl} and concludes the proof.
\qed
\end{proof}

\diff{In the following lemma, for each episode $m$ we bound the probability of \eqref{eq:cond-ukl} or \eqref{eq:cond-lkl} by a desired confidence level $\delta_m$, whose choice we postpone until the end of this proof. For now, we simply assume that we picked a function $f$ that satisfies $f(m)\log (m) e^{-f(m)} = O(\delta_m)$. We also denote $\Delta_M = \sum_{m=1}^{M}\delta_m$.}

\begin{lemma}[Boundary crossing probability]
\label{lemma:boundary-crossing-prob}
The following holds true, for any $1 \leq h \leq L$ and $m \leq M$,
\begin{equation*}
    \probability{\text{\eqref{eq:cond-ukl} or \eqref{eq:cond-lkl} is true}} = \diff{O((L+h)\delta_m)}
\end{equation*}
\end{lemma}
\begin{proof}
Since $V \leq \sum_{t=1}^h \gamma^t \mu(a^*_{1:t}) + \frac{\gamma^{h+1}}{1-\gamma}$, we have,
\begin{align*}
   \probability{\eqref{eq:cond-ukl}} &=  \probability{U_{a^*}(m) \leq V}\\
    &= \probability{\sum_{t=1}^L \gamma^t U^{\mu}_{a^*_{1:t}}(m) \leq \sum_{t=1}^L \gamma^t \mu(a^*_{1:t})}\\
    &\leq \probability{\exists 1\leq t \leq L : U^{\mu}_{a^*_{1:t}}(m) \leq \mu(a^*_{1:t})} \\
    &\leq \sum_{t=1}^L\probability{U^{\mu}_{a^*_{1:t}}(m) \leq \mu(a^*_{1:t})}
\end{align*}

In order to bound this quantity, \diff{we reduce the question to the application of a deviation inequality}. For all $1\leq t\leq L$, we have on the event $\{U^{\mu}_{a^*_{1:t}}(m) \leq \mu(a^*_{1:t})\}$ that $\hat{\mu}_{a^*_{1:t}}(m) \leq U^{\mu}_{a^*_{1:t}}(m) \leq \mu(a^*_{1:t}) < 1$. Therefore, for all $0 < \delta < 1 - \mu(a^*_{1:t})$, by definition of $U^{\mu}_{a^*_{1:t}}(m)$:

\begin{equation*}
    d(\hat{\mu}_{a^*_{1:t}}(m), U^{\mu}_{a^*_{1:t}}(m)+\delta) > \frac{f(m)}{T_{a^*_{1:t}}(m)}
\end{equation*}

As $d$ is continuous on $(0,1)\times[0, 1]$, we have by letting $\delta \leftarrow 0$ that:

\begin{equation*}
    d(\hat{\mu}_{a^*_{1:t}}(m), U^{\mu}_{a^*_{1:t}}(m)) \geq \frac{f(m)}{T_{a^*_{1:t}}(m)}
\end{equation*}

Since d is non-decreasing on $[\hat{\mu}_{a^*_{1:t}}(m), \mu(a^*_{1:t})]$,

\begin{equation*}
    d(\hat{\mu}_{a^*_{1:t}}(m), \mu(a^*_{1:t})) \geq d(\hat{\mu}_{a^*_{1:t}}(m), U^{\mu}_{a^*_{1:t}}(m)) \geq \frac{f(m)}{T_{a^*_{1:t}}(m)}
\end{equation*}

We have thus shown the following inclusion:

\begin{equation*}
    \{U^{\mu}_{a^*_{1:t}}(m) \leq \mu(a^*_{1:t})\} \subseteq \left\{ \mu(a^*_{1:t}) > \hat{\mu}_{a^*_{1:t}}(m) \text{ and } d(\hat{\mu}_{a^*_{1:t}}(m), \mu(a^*_{1:t})) \geq \frac{f(m)}{T_{a^*_{1:t}}(m)} \right\}
\end{equation*}

Decomposing according to the values of $T_{a^*_{1:t}}(m)$ yields:

\begin{equation*}
    \{U^{\mu}_{a^*_{1:t}}(m) \leq \mu(a^*_{1:t})\} \subseteq \bigcup_{n=1}^m \left\{ \mu(a^*_{1:t}) > \hat{\mu}_{a^*_{1:t}, n} \text{ and } d(\hat{\mu}_{a^*_{1:t}, n}, \mu(a^*_{1:t})) \geq \frac{f(m)}{n} \right\}
\end{equation*}

\diff{We now apply the deviation inequality} provided in Lemma 2 of Appendix A in \citep{Cappe2013}: $\forall \epsilon > 1$, provided that $0 < \mu(a^*_{1:t}) < 1$,

\begin{equation*}
\probability{\bigcup_{n=1}^m \left\{ \mu(a^*_{1:t}) > \hat{\mu}_{a^*_{1:t}, n} \text{ and } n \dber(\hat{\mu}_{a^*_{1:t}, n}, \mu(a^*_{1:t})) \geq \epsilon \right\}} \leq e\ceil{\epsilon \log m}e^{-\epsilon}\,.
\end{equation*}

By choosing $\epsilon = f(m)$, it comes
\begin{align*}
    \probability{\eqref{eq:cond-ukl}} &\leq \sum_{t=1}^L \diff{e\ceil{f(m)\log m}e^{-f(m)}} = \diff{O(L\delta_m)}
\end{align*}

The same reasoning gives: $\quad\displaystyle{
    \probability{\eqref{eq:cond-lkl}} = \diff{O(h\delta_m)}}$.
\qed
\end{proof}

\begin{lemma}[Confidence interval length and number of plays]
\label{lemma:ci-length}
Let $1 \leq h \leq L$, $a\in \mathcal{J}_h$ and $0 \leq h' < h$. Then  \eqref{eq:cond-dkl} is not satisfied if the following propositions are true:
\begin{equation}
\label{eq:sampled-enough-h}
   \forall 0\leq t\leq h', T_{a_{1:t}}(m) \geq \diff{2f(m)}(h+1)^2\gamma^{2(t-h-1)}
\end{equation}
and
\begin{equation}
\label{eq:sampled-enough}
   T_{a}(m) \geq \diff{2f(m)}(h+1)^2\gamma^{2(h'-h-1)}
\end{equation}
\end{lemma}
\begin{proof}
We start by providing an explicit upper-bound for the length of the confidence interval $U^{\mu}_{a_{1:t}} - L^{\mu}_{a_{1:t}}$. By Pinsker's inequality:
 
\begin{equation*}
    \diff{\dber(p, q) > \dquad(p, q)}
\end{equation*}

Hence for all $C>0$, 
\begin{equation*}
    \dber(p, q) \leq C   \implies 2(q - p)^2 < C  \implies p - \sqrt{C/2} < q < p + \sqrt{C/2}
\end{equation*}
And thus, for all $b\in A^*$, by definition of $U^{\mu}$ and $L^{\mu}$:
\begin{align*}
    U^{\mu}_{b}(m) - L^{\mu}_{b}(m) &\leq \frac{S_b(m)}{T_b(m)} + \sqrt{\frac{f(m)}{2T_b(m)}} -  \left(\frac{S_b(m)}{T_b(m)} - \sqrt{\frac{f(m)}{2T_b(m)}}\right) 
    = \diff{\sqrt{\frac{2f(m)}{T_b(m)}}}
\end{align*}

Now, assume that \eqref{eq:sampled-enough-h} and \eqref{eq:sampled-enough} are true. Then, we clearly have:
\begin{align*}
    \sum_{t=1}^h \gamma^t\left(U^{\mu}_{a_{1:t}}(m) - L^{\mu}_{a_{1:t}}(m)\right) &\leq \sum_{t=1}^{h'} \gamma^t \sqrt{\frac{2f(m)}{T_{a_{1:t}}(m)}} + \sum_{t=h'+1}^h \gamma^t \sqrt{\frac{2f(m)}{T_{a_{1:t}}(m)}} \\
    &\leq \frac{1}{(h+1)\gamma^{-h-1}} \sum_{t=1}^{h'} 1 + \frac{1}{(h+1)\gamma^{-h-1}} \sum_{t=h'+1}^h \gamma^{t-h'}  \\
    &\leq \frac{\gamma^{h+1}}{h+1} \left( h' + \frac{\gamma}{1-\gamma} \right)\leq \frac{\gamma^{h+1}}{1-\gamma}\,.\qquad\qquad\qed
\end{align*}
\end{proof}

\begin{lemma}
\label{lemma:size_Ph}
Let $1 \leq h \leq L,  a\in \mathcal{J}_h$ and $0\leq h'<h$. Then $\tau^a_{h,h'}=1$ implies that either equation \eqref{eq:cond-ukl} or \eqref{eq:cond-lkl} is satisfied or the following proposition is true:

\begin{equation}
    \label{eq:P-min-size}
    \exists 1\leq t \leq h': |\mathcal{P}_{h,h'}^{a_{1:t}}(m)| < \gamma^{2(t-h')}
\end{equation}
\end{lemma}

\begin{lemma}
\label{lemma:expected-P-size}
Let $1\leq h\leq L$ and $0 \leq h' < h$. Then the following holds true,

\begin{equation*}
    \expectedvalue |\mathcal{P}^\emptyset_{h,h'}(M)| = \tilde{O}\left(\gamma^{-2h'}\mathbbm{1}_{h'>0}\sum_{t=0}^{h'}(\gamma^2 \kappa')^t + (\kappa')^h\diff{\Delta_M} \right).
\end{equation*}
\end{lemma}

\begin{lemma}
\label{lemma:expected-plays-count}
Let $1\leq h\leq L$. The following holds true,
\begin{equation*}
    \expectedvalue \sum_{a\in \mathcal{J}_h} T_a(M) = \tilde{O}\left(\gamma^{-2h} + \diff{(\kappa')^h(1+M\Delta_M+\Delta_M) + (\kappa'\gamma^{-2})^h\Delta_M}\right)
\end{equation*}
\end{lemma}

Thus by combining Lemma \ref{lemma:expected-regret} and \ref{lemma:expected-plays-count} we obtain:
\begin{equation*}
    \expectedvalue r_n = \tilde{O}\left(\gamma^H + \gamma^{-H}M^{-1} + (\kappa' \gamma)^{H}M^{-1}\diff{(1+M\Delta_M+\Delta_M)} + (\kappa')^H\gamma^{-H}M^{-1}\diff{\Delta_M}\right)
\end{equation*}
Finally,
\begin{itemize}
    \item if $\kappa'\gamma^2 \leq 1$, we take $H = \floor{\log M / (2\log 1/\gamma)}$ to obtain:
    \begin{equation*}
        \expectedvalue r_n = \tilde{O}\left(M^{-\frac{1}{2}} + M^{-\frac{1}{2}} + M^{-\frac{1}{2}}M^{\frac{\log \kappa'}{2\log 1/\gamma}} \diff{\Delta_M}\right)
    \end{equation*}
    For the last term to be of the same order of the others, we need to have $\Delta_M = O(M^{-\frac{\log \kappa'}{2\log 1/\gamma}})$. Since $\kappa'\gamma^2 \leq 1$, we achieve this by taking \diff{$\Delta_M = O(M^{-1})$}.
    \item if $\kappa'\gamma^2 > 1$, we take $H = \floor{\log M / \log \kappa'}$ to obtain:
    \begin{equation*}
        \expectedvalue r_n = \tilde{O}\left(M^{\frac{\log \gamma}{\log \kappa'}} + M^{\frac{\log \gamma}{\log \kappa'}}\diff{(1+M\Delta_M+\Delta_M)} + M^{\frac{\log 1/\gamma}{\log \kappa'}}\diff{\Delta_M}\right)
    \end{equation*}
    Since $\kappa'\gamma^2 > 1$, the dominant term in this sum is $M^{\frac{\log \gamma}{\log \kappa'}}M\Delta_M$. Again, taking \diff{$\Delta_M = O(M^{-1})$} yields the claimed bounds.
\end{itemize}
Thus, the claimed bounds are obtained in both cases as long as we can impose $\Delta_M = O(M^{-1})$, that is, find a sequence $(\delta_m)_{1\leq m\leq M}$ and a function $f$ verifying:
\begin{equation}
    \sum_{m=1}^M \delta_m = O(M^{-1})\quad \text{and}\quad f(m)\log (m) e^{-f(m)} = O(\delta_m)
\end{equation}

By choosing \diff{$\delta_m = M^{-2}$ and $f(m) = 2 \log M + 2 \log\log M$}, the corresponding \KLOLOP algorithm does achieve the regret bound claimed in Theorem \ref{thm:regret}.

\section{Experiments}
\label{sec:experiments}

We have performed some numerical experiments to evaluate and compare the following planning algorithms\footnote[1]{The source code is available at \url{https://eleurent.github.io/kl-olop/}}:
\begin{itemize}
    \item \texttt{Random}: returns an action at random, we use it as a minimal performance baseline.
    \item \texttt{OPD}: the \emph{Optimistic Planning for Deterministic systems} from \citep{Hren2008}, used as a baseline of optimal performance. This planner is only suited for deterministic environments, and exploits this property to obtain faster rates. However, it is expected to fail in stochastic environments.
    \item \OLOP: as described in section \ref{sec:kl-olop-olop}.\footnote[2]{Note that we use the lazy version of $\OLOP$ and $\KLOLOP$ presented in Section \ref{sec:time-complexity}, otherwise the exponential running-time would have been prohibitive.}
    \item \KLOLOP: as described in section \ref{sec:kl-olop-kl-olop}.\footnotemark[2]
    \item \texttt{KL-OLOP}(1): an aggressive version of \KLOLOP where we used $f_1(m) = \log M$ instead of $f_2(m)$. This threshold function makes the upper bounds even tighter, at the cost of an increase probability of violation. Hence, we expect this solution to be more efficient in close-to-deterministic environments. However, since we have no theoretical guarantee concerning its regret as we do with $\KLOLOP$, it might not be conservative enough and converge too early to a suboptimal sequence, especially in highly stochastic environments.
\end{itemize}

They are evaluated on the following tasks, using a discount factor of $\gamma=0.8$:
\begin{itemize}
    \item A \href{https://github.com/eleurent/highway-env/}{highway driving} environment \citep{Leurent2018}: a vehicle is driving on a road randomly populated with other slower drivers, and must make their way as fast as possible while avoiding collisions by choosing on the the following actions: \texttt{change-lane-left}, \texttt{change-lane-right}, \texttt{no-op}, \texttt{faster}, \texttt{slower}.
    \item A \href{https://github.com/maximecb/gym-minigrid}{gridworld} environment \citep{gym_minigrid}: the agent navigates in a randomly-generated gridworld composed of either empty cells, terminal lava cells, and goal cells where a reward of $1$ is collected at the first visit.
    \item A stochastic version of the gridworld environment with noisy rewards, where the noise is modelled as a Bernoulli distribution with a 15\% probability of error, i.e. receiving a reward of 1 in an empty cell or 0 in a goal cell.
\end{itemize}

\begin{figure}[htp]
\centering

\subfloat{
	\includegraphics[width=\textwidth]{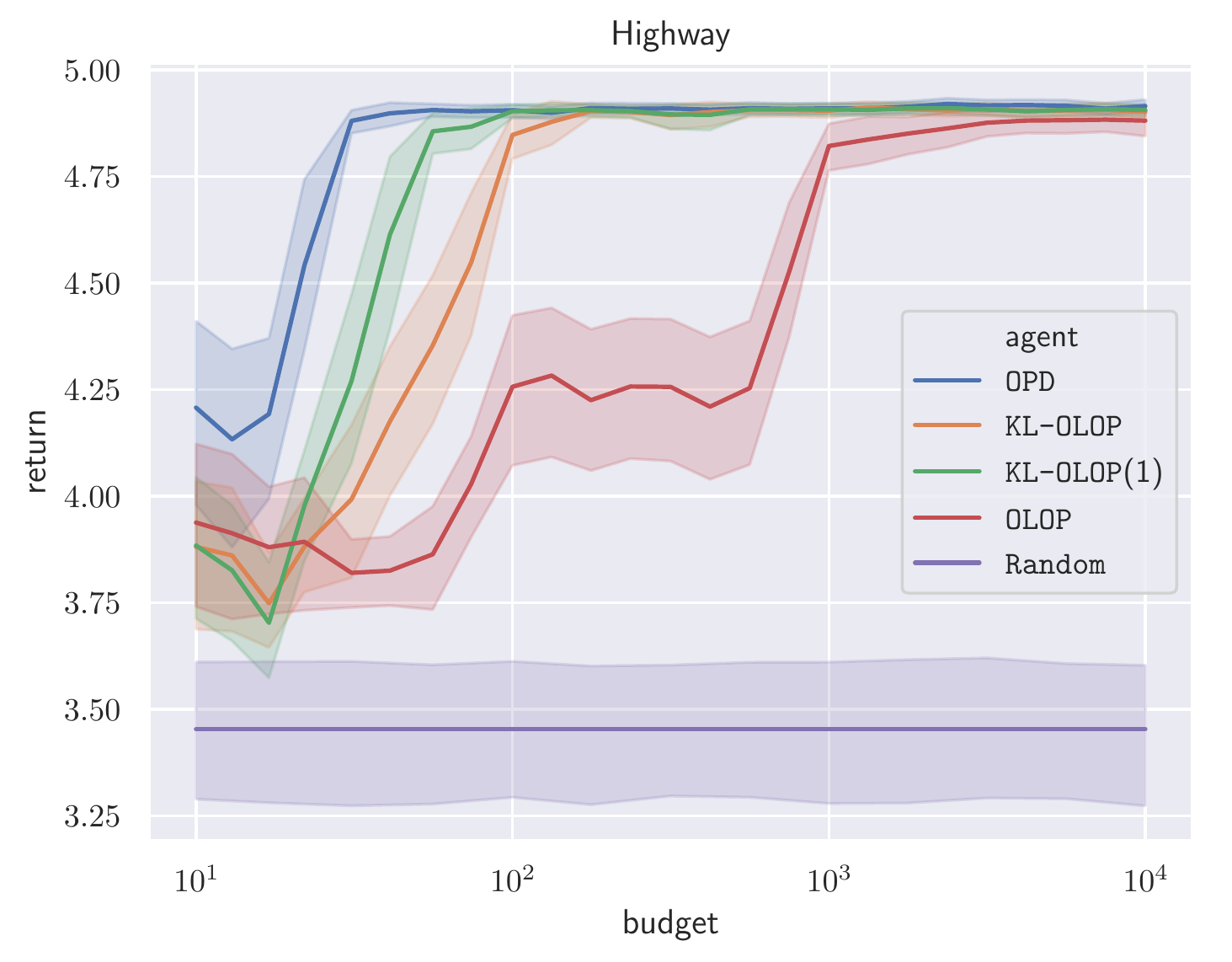}
	\label{sub:highway}
}
\newline
\subfloat{
	\includegraphics[width=0.5\textwidth]{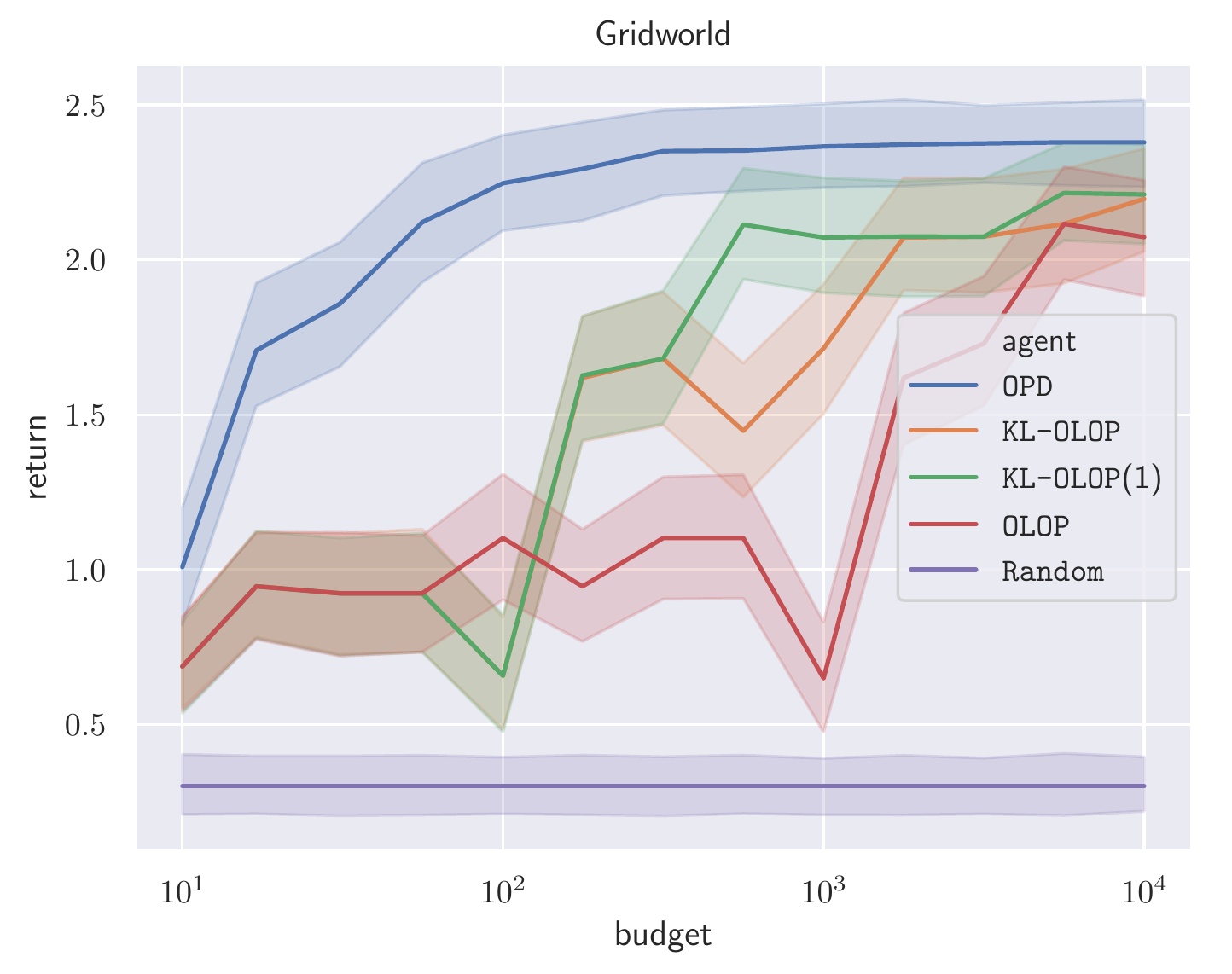}
	\label{sub:gridworld}
}
\subfloat{
	\includegraphics[width=0.5\textwidth]{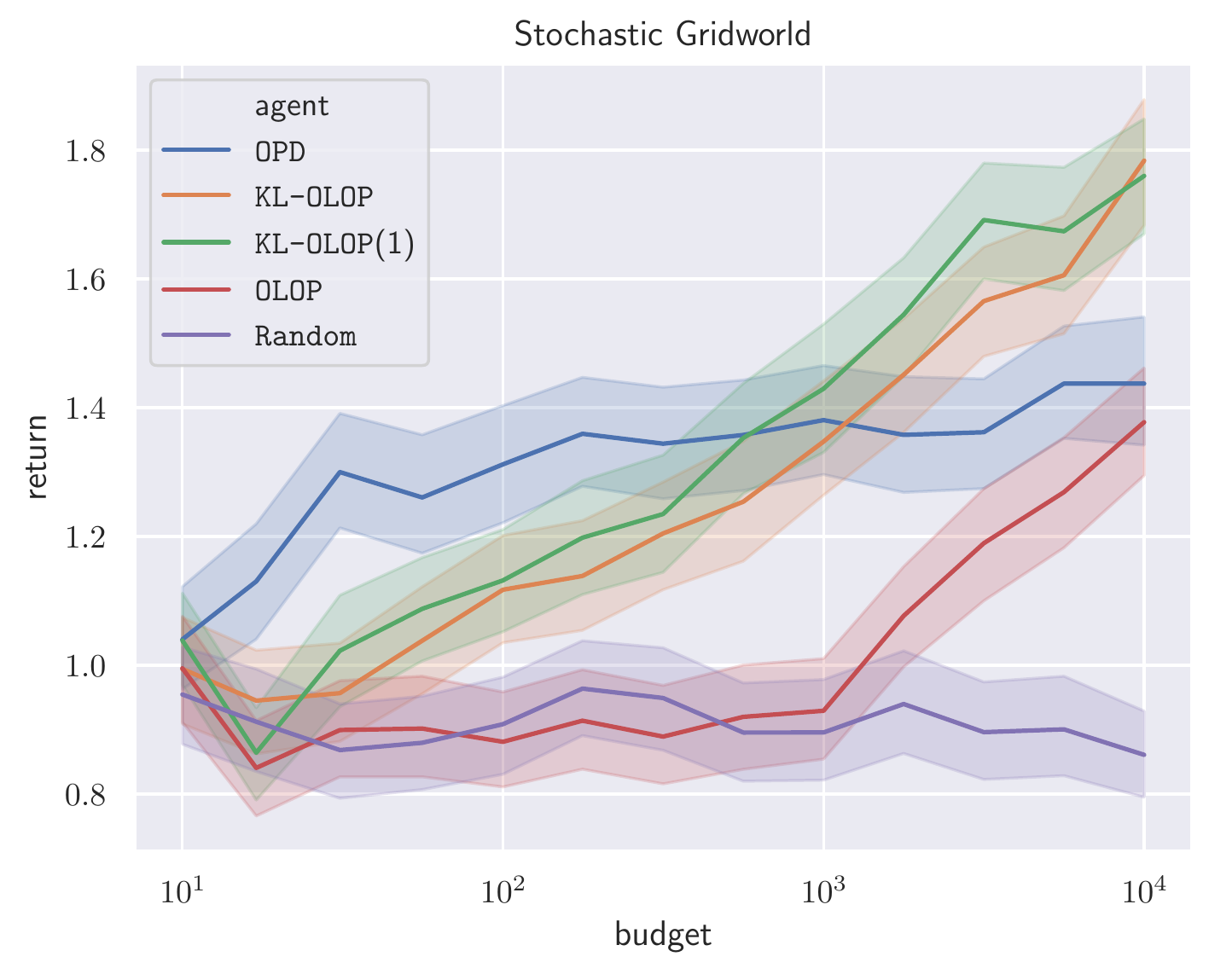}
	\label{sub:gridworld_stoch}
}
    \caption{Numerical experiments: for each environment-agent configuration, we compute the average return over 100 runs –- along with its 95\% confidence interval –- with respect to the available budget $n$.}
    \label{fig:experiments}
\end{figure}

The results of our experiments are shown in Figure \ref{fig:experiments}. The \texttt{ODP} algorithm converges very quickly to the optimal return in the two first environments, shown in Figure \ref{sub:highway} and Figure \ref{sub:gridworld}, because it exploits their deterministic nature: it needs neither to estimate the rewards through upper-confidence bounds nor to sample whole sequences all the way from the root when expanding a leaf, which provides a significant speedup. It can be seen of an oracle allowing to measure the conservativeness of stochastic planning algorithms. And indeed, even before introducing stochasticity, we can see that \OLOP performs quite badly on the two environments, only managing to solve them with a budget in the order of $10^{3.5}$. In stark contrast, \KLOLOP makes a much better use of its samples and reaches the same performance an order of magnitude faster. Examples of expanded trees are showcased in the Supplementary Material. Furthermore, in the stochastic gridworld environment shown in Figure \ref{sub:gridworld_stoch}, we observe that the deterministic \texttt{ODP} planner's performance saturates as it settles to suboptimal trajectories, as expected. Conversely, the stochastic planners all find better-performing open-loop policies, which justifies the need for this framework. Again, \KLOLOP converges an order of magnitude faster than \OLOP. Finally, \texttt{KL-OLOP}(1) enjoys good performance overall and displays the most satisfying trade-off between aggressiveness in deterministic environments and conservativeness in stochastic environments; hence we recommend this tuning for practical use.

\section{Conclusion}

We introduced an enhanced version of the \OLOP algorithm for open-loop online planning, whose design was motivated by an investigation of the over-conservative search behaviours of \OLOP. We analysed its sample complexity and showed that the original regret bounds are preserved, while its empirical performances are increased by an order of magnitude in several numerical experiments. Finally, we proposed an efficient implementation that benefits from a substantial speedup, facilitating its use for real-time planning applications.

\bibliographystyle{splncs04}

\newpage
\appendix
\begin{center}
    \Large Supplementary Material
\end{center}

\section{Examples of expanded trees}

The trees expanded by different planning algorithms in the Highway environment are displayed in Figure \ref{fig:trees}. They were all generated in the same initial root state, under the same budget of $n=10^3$ calls to the generative model. We observe that \ODP exploits the deterministic setting efficiently and produces a very sparse tree densely concentrated around the optimal trajectory. Conversely, the tree developed by the \OLOP algorithm is quite evenly balanced, which suggests that \OLOP behaves as uniform planning as hypothesised in \ref{sec:kl-olop-behaviour}. In stark contrast, \KLOLOP is much more efficient and produced a highly unbalanced tree, exploring the same regions of the tree as \ODP. 

\begin{figure}[pht]
    \centering

    \includegraphics[width=0.64\textwidth]{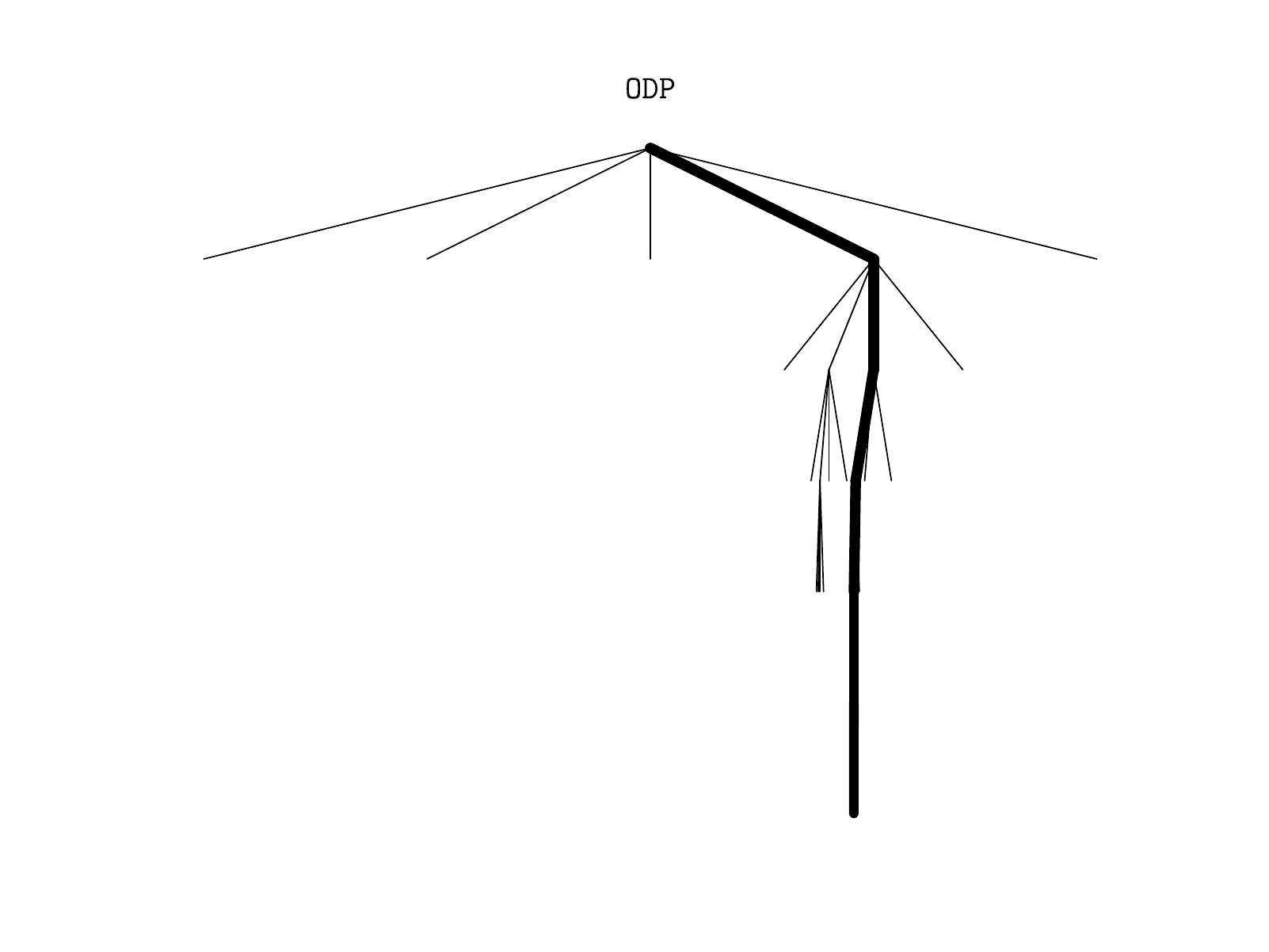}    \includegraphics[width=0.64\textwidth]{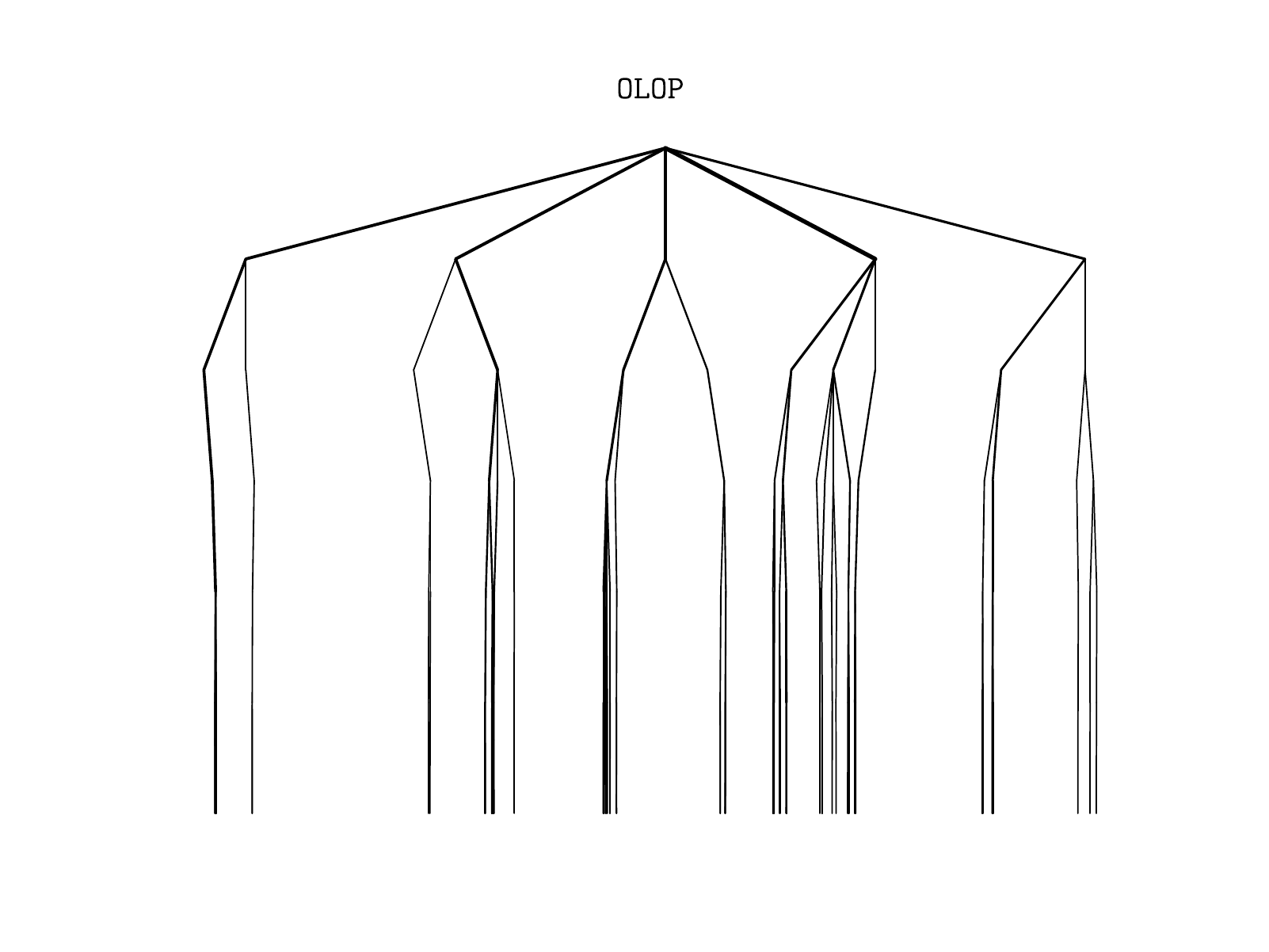}
    \includegraphics[width=0.64\textwidth]{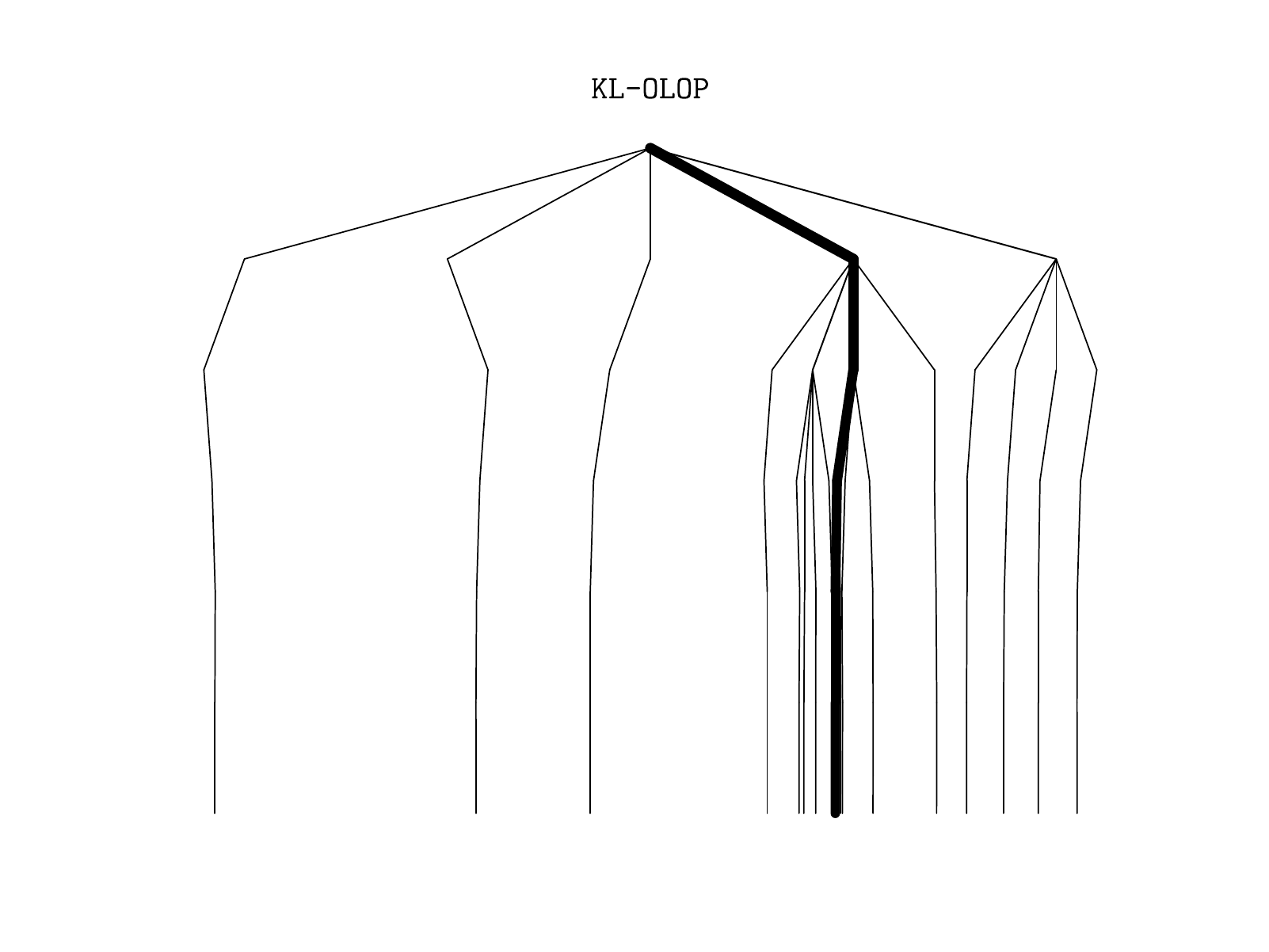}

    \caption{The look-ahead trees expanded by different planning algorithms for a budget $n=10^3$, shown down to depth 6. The width of edges is proportional to the nodes visit count $T_a(M)$.}
    \label{fig:trees}
\end{figure}

\section{Supplementary proofs of Theorem \ref{thm:regret}}

\subsection{Lemma \ref{lemma:expected-regret}}
\begin{proof}
The proof is identical to that of Lemma 4 in \citep{Bubeck2010}.
Since $\arg \max _{a \in A} T_{a}(M)$, and $\sum_{a \in A} T_{a}(M)=M,$ we have $T_{a(n)}(M) \geq M / K$, and thus:
\begin{equation*}
\frac{M}{K}(V-V(a(n))) \leq(V-V(a(n))) T_{a(n)}(M) \leq \sum_{m=1}^{M} V-V\left(a^{m}\right)
\end{equation*}
Hence, we have, $r_{n} \leq \frac{K}{M} \sum_{m=1}^{M} V-V\left(a^{m}\right)$. Now remark that, for any sequence of actions $a\in A^L$, we have either:
\begin{itemize}
    \item $a_{1 : H} \in \mathcal{I}_{H} ;$ which implies $V-V(a) \leq \frac{2 \gamma^{H+1}}{1-\gamma}$
    \item or there exists $1\leq h \leq H$ such that $a_{1:h} \in \mathcal{J}_h$; which implies $V-V(a) \leq V-V\left(a_{1 : h-1}\right)+\frac{\gamma^{h}}{1-\gamma} \leq \frac{3 \gamma^{h}}{1-\gamma}$.
\end{itemize}
Thus we can write:
\begin{equation*}
\begin{aligned} \sum_{m=1}^{M}\left(V-V\left(a^{m}\right)\right) &=\sum_{m=1}^{M}\left(V-V\left(a^{m}\right)\right)\left(\mathbbm{1}\left\{a^{m} \in \mathcal{I}_{H}\right\}+\mathbbm{1}\left\{\exists 1 \leq h \leq H : a_{1 : h}^{m} \in \mathcal{J}_{h}\right\}\right) \\ & \leq \frac{2 \gamma^{H+1}}{1-\gamma} M+3 \sum_{h=1}^{H} \sum_{a \in \mathcal{J}_{h}} \frac{\gamma^{h}}{1-\gamma} T_{a}(M) \end{aligned}
\end{equation*}
\qed
\end{proof}

\subsection{Lemma \ref{lemma:size_Ph}}
\begin{proof}
The event $\tau^a_{h,h'}=1$ implies $a^{m+1}\in a A^*$ and \eqref{eq:sampled-enough}. This implies by Lemma \ref{lemma:sub-optimal-pull} that either \eqref{eq:cond-ukl}, \eqref{eq:cond-lkl} or \eqref{eq:cond-dkl} is satisfied. Now by Lemma \ref{lemma:ci-length} this implies that either \eqref{eq:cond-ukl} is true or \eqref{eq:cond-lkl} is true or \eqref{eq:sampled-enough-h} is false. We now prove that if \eqref{eq:P-min-size} is not satisfied then \eqref{eq:sampled-enough} is true, which clearly ends the proof.
This follows from: For any $0 \leq t \leq h'$:

\begin{align*}
    T_{a_{1:t}}(m) &= \sum_{b\in a_{1:t}A^{h-t}} T_b(m) \geq \sum_{b\in \mathcal{P}^{a_{1:t}}_{h,h'}} T_b(m) \\
    &\geq \left(\gamma^{2(t-h')}\right) \left(2f(m)(h+1)^2\gamma^{2(h'-h-1)}\right)\\
    &= 2f(m)(h+1)^2\gamma^{2(t-h-1)}\,.\qquad\qquad\qed
\end{align*}
\end{proof}

\subsection{Lemma \ref{lemma:expected-P-size}}

\begin{proof}
The proof is identical to that of Lemma 9 in \citep{Bubeck2010}.

\noindent
Let $h'\geq 1$ and $0 \leq s \leq h'$. We introduce the following random variables:
\begin{equation*}
m_{s}^{a}=\min \left(M, \min \left\{m \geq 0 :\left|\mathcal{P}_{h, h^{\prime}}^{a}(m)\right| \geq \gamma^{2\left(s-h^{\prime}\right)}\right\}\right).
\end{equation*}
We will prove recursively that,
\begin{equation}
\label{eq:toprove}
\left|\mathcal{P}_{h, h^{\prime}}^{\emptyset}(m)\right| \leq \sum_{t=0}^{s} \gamma^{2\left(t-h^{\prime}\right)}\left|\mathcal{I}_{t}\right|+\sum_{a \in \mathcal{I}_{s}}\left|\mathcal{P}_{h, h^{\prime}}^{a} \setminus \cup_{t=0}^{s} \mathcal{P}_{h, h^{\prime}}^{a_{1:t}}\left(m_{t}^{a_{1:t}}\right)\right|
\end{equation}
The result is true for $s = 0$ since $\mathcal{I}_0 = \{\emptyset\}$ and by definition of $m^\emptyset_0$,
\begin{equation*}
\left|\mathcal{P}_{h, h^{\prime}}^{\emptyset}(m)\right| \leq \gamma^{-2 h^{\prime}}+\left|\mathcal{P}_{h, h^{\prime}}^{\emptyset}(m) \setminus \mathcal{P}_{h, h^{\prime}}^{\emptyset}\left(m_{0}^{\emptyset}\right)\right|
\end{equation*}
Now let us assume that the result is true for $s<h'$. We have:
\begin{align*}
\sum_{a \in \mathcal{I}_{s}}\left|\mathcal{P}_{h, h^{\prime}}^{a}(m) \setminus \cup_{h, h^{\prime}}^{a_{1 : t}}\left(m_{t}^{a_{1 : t}}\right)\right|&=\sum_{a \in \mathcal{I}_{s+1}}\left|\mathcal{P}_{h, h^{\prime}}^{a}(m) \setminus \cup_{t=0}^{s} \mathcal{P}_{h, h^{\prime}}^{a_{1 : t}}\left(m_{t}^{a_{1 : t}}\right)\right|\\
&\leq \sum_{a \in \mathcal{I}_{s+1}} \gamma^{2\left(s+1-h^{\prime}\right)}+\left|\mathcal{P}_{h, h^{\prime}}^{a}(m) \setminus \cup_{t=0}^{s+1} \mathcal{P}_{h, h^{\prime}}^{a_{1 : t}}\left(m_{t}^{a_{1 : t}}\right)\right|\\
&= \gamma^{2\left(s+1-h^{\prime}\right)}\left|\mathcal{I}_{s+1}\right|+\sum_{a \in \mathcal{I}_{s+1}}\left|\mathcal{P}_{h, h^{\prime}}^{a}(m) \setminus \cup_{t=0}^{s+1} \mathcal{P}_{h, h^{\prime}}^{a_{1 ; t}}\left(m_{t}^{a_{1 : t}}\right)\right|
\end{align*}
which ends the proof of \eqref{eq:toprove}. Thus we proved (by taking $s=h'$ and $m=M$):
\begin{equation*}
\begin{aligned}\left|\mathcal{P}_{h, h^{\prime}}^{\emptyset}(M)\right| & \leq \sum_{t=0}^{h^{\prime}} \gamma^{2\left(t-h^{\prime}\right)}\left|\mathcal{I}_{t}\right|+\sum_{a \in \mathcal{I}_{h^{\prime}}}\left|\mathcal{P}_{h, h^{\prime}}^{a}(M) \setminus \cup_{t=0}^{s+1} \mathcal{P}_{h, h^{\prime}}^{a_{1 : t}}\left(m_{t}^{a_{1 : t}}\right)\right.\\ &=\sum_{t=0}^{h^{\prime}} \gamma^{2\left(t-h^{\prime}\right)}\left|\mathcal{I}_{t}\right|+\sum_{a \in \mathcal{J}_{h}}\left|\mathcal{P}_{h, h^{\prime}}^{a}(M) \setminus \cup_{h, h^{\prime}}^{a_{1 : t}}\left(m_{t}^{a_{1 : t}}\right)\right| \end{aligned}
\end{equation*}

Now, for any $a\in \mathcal{J}_h$, let $\tilde{m} = \max_{0\leq t\leq h'} m_t^{a_{1:t}}$. Note that for $m\geq \tilde{m}$, equation \eqref{eq:P-min-size} is not satisfied. Thus we have
\begin{equation*}
\begin{aligned}
\left|\mathcal{P}_{h, h^{\prime}}^{a} \setminus \cup_{h, h^{\prime}}^{s+1} \mathcal{P}_{h, h^{\prime}}^{a_{1 : t}}\left(m_{t}^{a_{1 : t}}\right)\right|=&\sum_{m=\tilde{m}}^{M-1} \tau_{h, h^{\prime}}^{a}(m+1)=\sum_{m=0}^{M-1} \tau_{h, h^{\prime}}^{a}(m+1) \mathbbm{1}\{\eqref{eq:P-min-size} \text { is not satisfied }\} \\ & \leq \sum_{m=0}^{M-1} \tau_{h, h^{\prime}}^{a}(m+1) \mathbbm{1}\{\eqref{eq:cond-ukl} \text { or }\eqref{eq:cond-lkl}\} \end{aligned}
\end{equation*}
where the last inequality results from Lemma \ref{lemma:size_Ph}. Hence, we proved:

\begin{equation*}
\left|\mathcal{P}_{h, h^{\prime}}^{\emptyset}\right| \leq \sum_{t=0}^{h^{\prime}} \gamma^{2\left(t-h^{\prime}\right)}\left|\mathcal{I}_{t}\right|+\sum_{m=0}^{M-1} \sum_{a \in \mathcal{J}_{h}} \mathbbm{1}\{\eqref{eq:cond-ukl}\text{ or }\eqref{eq:cond-lkl}\}
\end{equation*}

Taking the expectation and applying Lemma \ref{lemma:boundary-crossing-prob} yield the claimed bound for $h'\geq 1$.

Now for $h' = 0$ we need a modified version of Lemma \ref{lemma:size_Ph}. Indeed in this case one can directly prove that $\tau_{h,0}^a(m+1)=1$ implies that either equation \eqref{eq:cond-ukl} or \eqref{eq:cond-lkl} is satisfied (this follows from the fact that $\tau_{h,0}^a(m+1)=1$ always imply that \eqref{eq:sampled-enough-h} is true for $h'= 0$). Thus we obtain:
\begin{equation*}
\left|\mathcal{P}_{h, h^{\prime}}^{\emptyset}\right|=\sum_{m=0}^{M-1} \sum_{a \in \mathcal{J}_{h}} \tau_{h, 0}^{a}(m+1) \leq \sum_{m=0}^{M-1} \sum_{a \in \mathcal{J}_{h}} \mathbbm{1}\{\eqref{eq:cond-ukl}\text{ or }\eqref{eq:cond-lkl}\}
\end{equation*}
Taking the expectation and applying Lemma \ref{lemma:boundary-crossing-prob} yield the claimed bound for $h' = 0$ and ends the proof.
\qed
\end{proof}

\subsection{Lemma \ref{lemma:expected-plays-count}}

\begin{proof}
The proof is identical to that of Lemma 10 in \citep{Bubeck2010}:
\begin{align*}
\sum_{a \in \mathcal{J}_{h}} T_{a}(M) &= \sum_{a \in \mathcal{J}_{h} \backslash \mathcal{P}_{h, h-1}^{\emptyset}} T_{a}(M)+\sum_{h^{\prime}=1}^{h-1} \sum_{a \in \mathcal{P}_{h, h^{\prime}}^{\emptyset} \setminus \mathcal{P}_{h, h^{\prime}-1}^{\emptyset}} T_{a}(M)+\sum_{a \in \mathcal{P}_{h, 0}^{\emptyset}} T_{a}(M)\\
&\leq 2f(m)(h+1)^{2} \gamma^{2(h-2-h)}\left|\mathcal{J}_{h}\right|+\sum_{h^{\prime}=1}^{h-1} 2f(m)(h+1)^{2} \gamma^{2\left(h^{\prime}-2-h\right)} \log M\left|\mathcal{P}_{h, h^{\prime}}^{\emptyset}\right|+M\left|\mathcal{P}_{h, 0}^{\emptyset}\right|\\
&=\tilde{O}\left(\left(\kappa^{\prime}\right)^{h}+\gamma^{-2 h} \sum_{h^{\prime}=1}^{h-1} \gamma^{2 h^{\prime}}\left|\mathcal{P}_{h, h^{\prime}}^{\emptyset}\right|+M\left|\mathcal{P}_{h, 0}^{\emptyset}\right|\right)
\end{align*}
Taking the expectation and applying the bound of Lemma \ref{lemma:expected-P-size} give the claimed bound.
\qed
\end{proof}

\section{Proof of Theorem \ref{thm:consistency}}

\begin{proof}

To prove consistency of Algorithm \ref{algo:lazy-kl-olop}, we need to show that the sequences of actions $a^m$ sampled at every episode are chosen arbitrarily from the same sets as in Algorithm \ref{algo:lazy-kl-olop}.
Namely, 

\begin{equation*}
    \left\{ b\in a A^{L-|a|} : a \in \argmax_{a\in \LL_{m-1}^+} B_a(m-1)\right\} = \argmax_{a\in A^L} B_a(m-1)
\end{equation*}

To that end, we first introduce some useful notations:

\begin{paragraph}{Definition}
Let $\Tau_m$ be the set of visited nodes after episode $m$:
\begin{equation*}
    \Tau_m \eqdef \left\{a\in A^*: T_a(m) > 0\right\}
\end{equation*}

We also define its extension $\Tau_m^+$ of visited nodes and their children:
\begin{equation*}
    \Tau_m^+ \eqdef \Tau_m + \Tau_m A
\end{equation*}

Now for $a\in A^*$, $\pi_m(a)$ (resp. $\pi_m^+(a)$) refers to its longest prefix within $\Tau_m$ (resp. $\Tau_m^+$):
\begin{align*}
    \pi_m(a) &\eqdef \argmax_{b\in\Tau_m} \{|b|: a\in b A^* \} \\
    \pi_m^+(a) &\eqdef \argmax_{b\in\Tau_m^+} \{|b|: a\in b A^* \}
\end{align*}

Finally, $\LL_m$ and $\LL_m^+$ are the image of $A^L$ by $\pi_m$ and $\pi_m^+$, respectively.
\begin{align*}
    \LL_m &\eqdef \pi_m(A^L) \\
    \LL_m^+ &\eqdef \pi_m^+(A^L) \}
\end{align*}
\end{paragraph}

\begin{remark}[About children extensions]
We could frame Algorithm \ref{algo:lazy-kl-olop} in terms of $\Tau_m$ and $\LL_m$, for which mathematical proofs are more straight-forward. However, the iterative construction of $\LL_m$ is tricky and it would require inverting $\pi_m$ on $\LL_m$ which is non-trivial. On the contrary, introducing their extensions  $\Tau_m^+$ and $\LL_m^+$ slightly complicates the proof, but greatly simplifies the construction of $\LL_m^+$ and the computation of ${\pi_m^+}^{-1}$ on $\LL_m^+$, which is why we use these sets in practice.
\end{remark}

\begin{lemma}[Sets construction]
$\Tau_m^+$ and $\LL_m^+$ are indeed the sets computed in Algorithm  \ref{algo:lazy-kl-olop}.
\end{lemma}
\begin{proof}
Note that for each episode $1 \leq m \leq M - 1$, we have:
\begin{equation}
\label{eq:tau_mp1}
    \Tau_{m+1} = \Tau_{m} + \sum_{t=0}^L a^{m+1}_{1:t}
\end{equation}
Indeed, the nodes visited at least once at time $m+1$ where either already visited once at time $m$ (e.g. in $\Tau_{m}$) or have been visited for the first time during episode $m+1$, which means they are a prefix of $a^{m+1}$. The reverse is clearly true as well.

This enables to write:
\begin{align*}
    \Tau_{m+1}^+ &= \Tau_{m+1} + \Tau_{m+1}A & \text{by definition}\\
    &= \Tau_{m} + \sum_{t=0}^L a^{m+1}_{1:t} + (\Tau_{m} + \sum_{t=0}^L a^{m+1}_{1:t})A & \text{by \eqref{eq:tau_mp1}} & \\
    &= (\Tau_{m} + \Tau_{m}A) + \sum_{t=0}^L a^{m+1}_{1:t} + \sum_{t=0}^L a^{m+1}_{1:t}A & \\
    &= \Tau_{m}^+ + a^{m+1}_{1:0} + \sum_{t=0}^L a^{m+1}_{1:t}A &\text{ as } \sum_{t=1}^L a^{m+1}_{1:t} \subset \sum_{t=0}^L a^{m+1}_{1:t}A\\
    &=  \Tau_{m}^+ + \sum_{t=0}^L a^{m+1}_{1:t}A  &\text{ as }a^{m+1}_{1:0}=\emptyset \in \Tau_{0} \subset \Tau_{m} \subset \Tau_m^+
\end{align*}
This recursion is the one implemented in Algorithm \ref{algo:lazy-kl-olop}: at each episode $m$, we add to $\Tau_{m}^+$ the children of the nodes along the sampled action sequence $a^{m}$.

Finally, we highlight that $\LL_m^+ = \pi^+(A^L)$ is the set of leaves of $\Tau_{m}^+$.
Indeed, nodes of $\LL_m^+$ belong to $\Tau_{m}^+$, but they cannot have a child in $\Tau_{m}^+$ as it would contradict the definition of $\LL_m^+$. Conversely, any leaf $a$ of $\Tau_{m}^+$ can be continued arbitrarily to a sequence $b$ of $A^L$, which  $a = \pi_m^+(b) \in \pi^+(A^L) = \LL_m^+$.

Thus, when updating $\Tau_{m-1}^+$, the set of its leaves is updated accordingly: when the children of a leaf $a^m_{1:t-1}$ are added to $\Tau_m^+$, they become new leaves in place of their parent. Hence, they are added to $\LL_m^+$ while $a^m_{1:t-1}$ is removed from it.
\qed
\end{proof}

\begin{lemma}[U-values conservation]
\label{lemma:value-conservation}
For all $a \in A^*$,
\begin{equation*}
    U_a(m) = U_{\pi_m(a)}(m) = U_{\pi_m^+(a)}(m)
\end{equation*}
\end{lemma}
\begin{proof}
Let $a \in A^*$, denote $h=|a|$ and $h'=|\pi_m(a)|$.

By definition of $\pi_m(a)$, $0 \leq h' \leq h$, and
\begin{itemize}
    \item for $1\leq t \leq h'$, we have $a_{1:t} = {\pi_m(a)}_{1:t}$ ;
    \item for $h'+1\leq t \leq h$, we have $a_{1:t} \not \in \Tau_m$, hence $T_{a_{1:t}}(m) = 0$ and $U^{\mu}_{a_{1:t}}(m) = 1$.
\end{itemize}
Then,
\begin{align*}
    U_a(m) &= \sum_{t=1}^h \gamma^t U^{\mu}_{a_{1:t}}(m) + \frac{\gamma^{h+1}}{1-\gamma} \\
    &= \sum_{t=1}^{h'} \gamma^t U^{\mu}_{a_{1:t}}(m) + \sum_{t=h'+1}^h \gamma^t \underbrace{U^{\mu}_{a_{1:t}}(m)}_1 + \frac{\gamma^{h+1}}{1-\gamma} \\
    &= \sum_{t=1}^{h'} \gamma^t U^{\mu}_{{\pi_m(a)}_{1:t}}(m) + \frac{\gamma^{h'+1}}{1-\gamma} \\
    &= U_{\pi_m(a)}(m)
\end{align*}

Now, consider $\pi_m^+(a) \in \Tau_m^+$.
By definition, it belongs either to $\Tau_m$ or $\Tau_m A$.
\begin{itemize}
    \item If $\pi_m^+(a) \in \Tau_m$, then $\pi_m^+(a) = \pi_m(a)$ and $U_{\pi_m^+(a)}(m) = U_{\pi_m(a)}(m)$.
    \item Else, $\pi_m^+(a) \in \Tau_m A$ and $p(\pi_m^+(a)) = \pi_m(a)$.
    
    As $\pi_m^+(a) \not \in \Tau_m$, we have $T_{\pi_m^+(a)}(m) = 0$ and $U^{\mu}_{\pi_m^+(a)}(m) = 1$.
    This yields:
    
    \begin{equation*}
        U_{\pi_m^+(a)}(m) = \sum_{t=1}^{h'} \gamma^t U^{\mu}_{{\pi_m^+(a)}_{1:t}}(m) + \gamma^{h'+1} \underbrace{U^{\mu}_{{\pi_m^+(a)}}(m)}_1 + \frac{\gamma^{h'+2}}{1-\gamma} = U_{\pi_m(a)}(m)
    \end{equation*}
    
\end{itemize}
We showed that $U_{\pi_m^+(a)}(m) = U_{\pi_m(a)}(m)$, which concludes the proof.
\qed
\end{proof}

\begin{lemma}[Inverse projection]
\label{lemma:inverse-proj}
For all $a\in \LL_m^+$ of length $h\leq L$,
\begin{equation*}
    {\pi_m^+}^{-1}(a) = a A^{L-h}
\end{equation*}

This allows to easily pick a sequence inside ${\pi_m^+}^{-1}(a)$: just continue the sequence $a$ with a default action of $A$ (e.g. the first) until it reaches length $L$.
\end{lemma}
\begin{proof}
Let $a\in \LL_m^+$. 

By definition of $\pi_m^+$, any sequence in ${\pi_m^+}^{-1}(a)$ is a suffix of $a$ of length $L$, so we clearly have the direct inclusion ${\pi_m^+}^{-1}(a) \subset a A^{L-h}$.

Now for the other side: let $b\in a A^{L-h}$, i.e. $a=b_{1:h}$. We need to show that $\pi_m^+(b) = a$.
As $a\in \LL_m^+$, there exists $c\in A^L$ such that $\pi_m^+(c) = a$.
\begin{itemize}
    \item If h = L, then $b=a$, so $b \in \LL_m^+ \subset \Tau_m^+$, and hence $\pi_m^+(b)=b=a$.
    \item If h < L, we can show by contradiction that $a \not \in \Tau_m$. Indeed, if $a \in \Tau_m$, then $c_{1:h+1}$ is the child of a node of $\Tau_m$ and hence belongs to $\Tau_m^+$. But then, $c_{1:h+1}$ is a prefix of $c$ in $\Tau_m^+$ with greater length than $a$, which contradicts the definition of $a = \pi_m^+(c)$.
    
    Now, because $a \not \in \Tau_m$, it is also true for all suffixes of $a$, and in particular for $b_{1:t}$ with $h \leq t \leq L$. Indeed, we have $a^s_{1:t} = b_{1:t} \implies a^s_{1:h} = b_{1:h} = a$, so:
    \begin{equation*}
        T_{b_{1:t}}(m) = \sum_{s=1}^m \mathbbm{1}\{a^s_{1:t} = b_{1:t}\} \leq \sum_{s=1}^m \mathbbm{1}\{a^s_{1:h} = a\} = T_a(m) = 0
    \end{equation*}
    Hence, $b_{1:t} \not \in \Tau_m$ for all $h \leq t \leq L$, so in particular $b_{1:t} \not \in \Tau_m^+$ for all $h+1 \leq t \leq L$. Since $b_{1:h} = a \in \Tau_m^+$, $a$ is indeed the longest prefix of $b$ in $\Tau_m^+$, that is: $\pi_m^+(b) = a$.
\end{itemize}
We have shown the other side of the inclusion: $a A^{L-h} \subset {\pi_m^+}^{-1}(a)$, which entails that the two sets are in fact equal.
\qed
\end{proof}

We can now conclude our proof of Theorem \ref{thm:consistency}: at episode $m$, \KLOLOP samples a sequence of action $a^m$ within the set $\argmax_{a\in A^L} U_a(m)$. 
However, we have:

\begin{align*}
    \argmax_{c\in A^L} U_c(m) &= \argmax_{c\in A^L} U_{\pi_m^+(c)}(m) & \text{by Lemma \ref{lemma:value-conservation}} \\
    &= {\pi_m^+}^{-1}\left(\argmax_{a\in {\pi_m^+}(A^L)} U_{a}(m)\right) & \\
    &= \left\{b\in{\pi_m^+}^{-1}(a) : a\in \argmax_{a\in \LL_m^+} U_{a}(m)\right\} & \\
    &= \left\{b\in a A^{L-|a|} : a\in \argmax_{a\in \LL_m^+} U_{a}(m)\right\} & \text{by Lemma \ref{lemma:inverse-proj}} 
\end{align*}

Thus, at each episode the sequence of actions $a^m$ sampled by Algorithm \ref{algo:lazy-kl-olop} could have been sampled by Algorithm \ref{algo:kl-olop} as well.

In particular, if the arbitrary rule used to pick a sequence from a set is the same for the two algorithms, then the sampled sequences $a^m$ will be identical, will have the same visit count $T_{a^m}(m)$, and in the end the returned action $a(n)$ will be the same.
\qed
\end{proof}
\end{document}